\documentclass[]{article}
\usepackage{proceed2e}
\usepackage{natbib}
\usepackage{graphicx}
\usepackage{verbatim}
\usepackage{url}
\usepackage{subfigure}
\usepackage{algorithm, algorithmic}
\usepackage{amsmath, amsthm, amsfonts, amssymb}
\usepackage[usenames]{color}
\usepackage{tikz}
\usepackage{booktabs}
\usepackage{multirow}
\usepackage{rotating}
\usepackage{bm}
\usepackage{wrapfig}
\usepackage{multicol}
\usepackage{epsfig}

\newtheorem{theorem}{Theorem}

\title{Categorization Axioms for Clustering Results}

\author{ {\bf Jian YU } \\
	Beijing Key Lab
	of Traffic Data Analysis and Mining\\
	Beijing Jiaotong University,Beijing, China\\
	Email: jianyu@bjtu.edu.cn
	\And
	{\bf Zongben XU}  \\
	School of Mathematics and Statistics\\Xi'an Jiaotong University, Xi'an, China    \\
	Email: zbxu@mail.xjtu.edu.cn
}

\begin{document}

\maketitle

\begin{abstract}
	
Cluster analysis has attracted more and more attention in the field of machine learning and data mining. Numerous clustering algorithms have been proposed and are being developed due to diverse theories and various requirements of emerging applications. Therefore, it is very worth establishing an unified axiomatic framework for data clustering. In the literature, it is an open problem and has been proved very challenging. In this paper, clustering results are axiomatized by assuming that an proper clustering result should satisfy categorization axioms. The proposed axioms not only introduce classification of clustering results and inequalities of clustering results, but also are consistent with prototype theory and exemplar theory of categorization models in cognitive science. Moreover, the proposed axioms lead to three principles of designing clustering algorithm and cluster validity index, which follow many popular clustering algorithms and cluster validity indices.

\end{abstract}

\section{Introduction}

Categorization is a fundamental ability for people to cognise objects in the world~\citep{lako1987women}. It is always a wish that a computer can have categorization ability like human being. Generally speaking, when seeing the objects in the world, a computer can categorize them into some concepts, such as tree, sky, water, and cat, and so on. As for computer, categorization means that objects are grouped into categories, including classification and cluster analysis. Classification requires some objects with category label, cluster analysis needs no object with category label except that semi-supervised clustering needs some information about category label~\citep{basu2005semi}. In Big Data time, the number of unlabelled objects becomes more and more large so that cluster analysis plays a more important role as a pivotal part of exploratory data analysis. Up to now, cluster analysis can not be well formalized as~\citep{jain2010data,Everitt2011cluster,ackerman2012towards} pointed out, but it is well known that its elementary requirements are to divide $n$ objects into $c$ classes so that the objects in the same class are similar and the objects in different classes are dissimilar.

Data clustering has been extensively applied for diverse requirements, such as VLSI design~\citep{wei1989towards},topic discovery~\citep{hofmann1999probabilistic},community detection~\citep{newman2007mixture},image segmentation~\citep{goldberger2002unsupervised},collaborative filtering~\citep{hofmann2004latent},personalized recommendation~\citep{shepitsen2008personalized},and
so on. And many theories have been applied to design clustering algorithms, such as Gestalt theory~\citep{zahn1971graph},probability
theory~\citep{redner1984mixture},information theory~\citep{tishby1995information},rate-distortion theory~\citep{rose1990statistical},fuzzy logic~\citep{bezdek1974cluster},graph theory~\citep{wu1993optimal},cognitive science~\citep{shepard1979additive},matrix decomposition~\citep{xu2003document},game theory~\citep{pavan2007dominant},quantum mechanics~\citep{horn2001algorithm},etc. Therefore, it is worth developing an unified axiomatic framework for data clustering to deal with common clustering properties. An appropriated axiomatic framework for data clustering can help us understand the properties of clustering, and provide principles of designing clustering algorithm and cluster validity index.

There are three ways to deal with the axiomatization of cluster analysis in the literature. One way is to axiomatize clustering criterion (popular clustering criterion is objective function of a clustering algorithm). When admissible reformulation functions are taken as clustering criteria, ~\citep{karayiannis1999axiomatic} presented an axiomatic approach to soft learning vector quantization (LVQ) and clustering; when cost function is taken as clustering criterion,~\citep{puzicha2000theory} designed some kind of clustering algorithms to satisfy invariance (including permutation invariance, scale invariance, shift invariance), monotonicity and robustness for dissimilarity matrix. Second way is to consider a clustering algorithm as an a clustering function and axiomatize clustering function. When clustering function is defined as a mapping from a feature matrix to a partition matrix,~\citep{wright1973formalization} proposed twelve consistent axioms for clustering function, a few clustering algorithm is studied based on~\citep{wright1973formalization}'s clustering axioms as~\citep{wright1973formalization}'s clustering axioms seem to be too demanding for many clustering functions in the literature.  When clustering function is defined as a mapping from n objects with pairwise distances to a hard partition,~\citep{jardine1971mathematical} established an axiomatic framework of hierarchical clustering,~\citep{kleinberg2003impossibility} introduced a surprising but more highly influential impossibility theorem for clustering by three axioms: no clustering function satisfies scale invariance, richness and consistency. As~\citep{correa2013indication} pointed out, ~\citep{kleinberg2003impossibility}'s work is accepted as a rigorous proof of the difficulty in finding a unified framework for different clustering approaches. Hence, subsequent researches devoted to finding a possibility theorem for specific clustering algorithms by relaxing Kleinberg's axioms. ~\citep{zadeh2009uniqueness} proved that single linkage clustering algorithm can satisfy the relaxed Kleinberg's axioms. ~\citep{zadeh2010towards} has further characterized Single Linkage and Max-Sum based on relaxed Kleinberg's axioms.~\citep{correa2013indication} improved Kleinberg's results for specific clustering algorithms by considering the algorithm parameters.  Third way is to axiomatize clustering quality measure (cluster validity index),~\citep{ben2008measures} have proved that several clustering quality measures satisfy Kleinberg's axioms without impossibility.

Obviously, the above three ways of axiomatizing cluster analysis do not characterize clustering results but clustering criterion, clustering function or clustering quality measure respectively. In the literature, Wright~\citep{wright1973formalization} has
questioned what is the nature of the results of cluster analysis, but reduced it to cluster quality (cluster validity) issue. Von Luxburg and Ben-David~\citep{von2005towards} have stated that an interesting clustering result should have a large distance to predefined ¡°uninteresting¡± clustering results such as the trivial clustering results, and pointed out that it should be defined what a clustering result should not be and such a concept has not yet been studied from a theoretical point of view. Gollapudi, Kumar, and Sivakumar~\citep{gollapudi2006programmable} have made a try to define what constitutes an acceptable clustering result from a pairwise distance point of view but focused on finding some efficient algorithms. To the best of our knowledge, no works on axiomatizing clustering result has been done in the literature. In this paper, we devote to developing an axiomatic framework of qualitatively characterizing clustering results, consistent with human cognition.

The major contributions of this paper are as follows:

1) Based on the definition of cluster (dis)similarity mapping, representation of clustering result is proposed.

2) Clustering results are axiomatized by  three categorization axioms.

3) The proposed axioms can not only classify clustering results into proper clustering, overlapping clustering and improper clustering, but also classify partition into proper partition, overlapping partition and improper partition. Moreover, boundary set is theoretically defined based on sample separation axiom.

4) Sample separation axiom can follow several inequalities about clustering results.

5) The proposed axioms can follow three principles of developing clustering algorithm and three principles of designing cluster validity index,which can not only deduce many famous clustering algorithms such as C-means, Classification Maximum Likelihood, Model based clustering, but also review several cluster validity indices.

The rest of the paper is organized as below: In section 2, how to represent a clustering result is studied. In section 3, the categorization axioms for clustering results are introduced, and the properties of clustering results are studied based on the proposed axioms. In section 4, how to design a clustering algorithm is discussed, and some related clustering algorithms are discussed. Section 5 discusses the relation between separation axioms and cluster validity index. In the final, the conclusions and some discussions are presented.

\section{Representation of Clustering Results}

Data clustering usually consists of four parts: data representation, clustering criterion, clustering algorithm, cluster validity. As discussed in Section 1, there are some good works in the literature to devoting to axiomatization of clustering criterion, clustering function or clustering quality measure. However, little attention has been paid to axiomatizing clustering results based on data representation in the literature. In this paper, we will study how to represent a clustering result when the given data set $X=\{x_1,x_2,\cdots,x_n\}$ is grouped into $c$ clusters $X_1,X_2,\cdots,X_c$.

For a clustering result, it has two kinds of cluster representation: cognitive cluster representation and extensional cluster representation. Cognitive cluster representation is originated from intensional definition of cluster, which is denoted by $\underline{X}=\{\underline{X_1},\underline{X_2},\cdots,\underline{X_c}\}$, where $\underline{X_i}$ represents $X_i$. Extensional cluster representation is based on extensional definition, which is represented by $U=[u_{ik}]_{c\times n}$, where  $u_{ik}\geq 0$  represents the membership that $x_k$ belongs to cluster $\underline{X_i}$.  %Here, a cluster is equivalent to a category.

As for cognitive cluster representation, it needs to define what a cluster is in essence. Obviously, to define a cluster is equivalent to define a category, but~\citep{wittgenstein1953philosophical} had questioned the assumption that important concepts of categories could be  defined in a unified classical definition. In the following, cluster and category are interchangeable in this paper. As a classical definition requires a necessary and sufficient characteristics, what a cluster is cannot have a unified classical definition~\citep{Everitt2011cluster}. Cognitive scientists have developed different theories of categorization to represent categories such as prototype theory ~\citep{rosch1978principles}  and exemplar theory~\citep{medin1978context}, which provide the foundation for cognitive category representation.

According to prototype theory, a category can be represented by a prototype and a prototype is usually represented by features. Therefore, the cognitive category representation $\underline{X_i}$ for a cluster $X_i$ can be represented by $v_i=[v_{i1},v_{i2},\cdots,v_{i\tau}]$, $\tau$ is the dimensionality of cluster feature, which may be different from the dimensionality of the data set $X$. Hence, cognitive category representation of a clustering result $\underline{X}=\{\underline{X_1},\underline{X_2},\cdots,\underline{X_c}\}$ can be represented by a cluster feature matrix  $V=[v_{io}]_{c \times \tau}=[v_1^T,v_2^T,\cdots, v_c^T]^T$.

Not all cognitive category representation of clustering results of prototype based clustering algorithms can be represented by cluster feature matrix. In the literature, sometimes a prototype refers to an ideal object~\citep{murphy2004big}. For example,  prototype based clustering algorithm affinity propagation~\citep{frey2007clustering} assumes that a cluster is represented by one object, and its clustering result cannot be represented by cluster feature matrix.

According to exemplar theory, a category is represented by multiple exemplars, such a cognitive category representation is called exemplar cluster representation. In the literature of cluster analysis, exemplar cluster representation includes tree and graph connected component. Tree representation results in hierarchical clustering~\citep{sneath1957application}, a connected component represents a cluster for minimum cut~\citep{wei1989towards},Markov clustering algorithm~\citep{van2000graph},and support vector clustering~\citep{ben2002support}.  However, not all clustering algorithms have an explicit cognitive cluster representation, for example, pairwise data clustering~\citep{hofmann1997pairwise},Information based clustering~\citep{slonim2005information},dominant set clustering~\citep{pavan2007dominant}, and so on. Under such a case, we assume a latent cognitive cluster representation exists, which is still represented by $\underline{X}=\{\underline{X_1},\underline{X_2},\cdots,\underline{X_c}\}$.  In this paper, it is always supposed that a cognitive category representation exists for a cluster.

As for extensional cluster representation $U$, it gives only information about partition. In mathematical terms,   a clustering result can be represented by a partition matrix $U =[u_{ik}]_{c \times n}$. In the literature, various constraints on partition matrix result in different partitions as follows:

 {\bf Hard Partition:}, If $\sum_{i=1}^{c} u_{ik}=1, u_{ik} \in \{0,1 \}, \forall i, 1\leq \sum_{k=1}^{n} u_{ik} < n$, then $U_{h} =[u_{ik}]_{c \times n}$ is called hard partition.

 {\bf Soft Partition:}  If   $\sum_{i=1}^{c} u_{ik}=1, 0\leq u_{ik} \leq 1,$ and $ \forall i, 0 < \sum_{k=1}^{n} u_{ik} < n$. then $U_{s} =[u_{ik}]_{c \times n}$ is called soft partition.

{\bf Possibilistic Partition:} If  $\forall k, \sum_{i=1}^{c} u_{ik} > 0, u_{ik} \geq 0$, and $\forall 1 \leq i \leq c, \sum_{k=1}^{n} u_{ik} > 0$, then $U_{p} =[u_{ik}]_{c \times n}$ is called possibilistic partition.

 Depending on the type of partition, clustering algorithm can be classified into hard clustering and soft clustering.

{\bf Hard Clustering:} If a clustering algorithm outputs  a hard partition, it is called hard clustering.

{\bf Soft Clustering:} If  a clustering algorithm outputs a soft or a possibilistic partition, it is called soft clustering.

In theory, cognitive category representation should have categorization ability. According to cognitive science, categorization is based on the similarity of an object to the underlying cluster representation.  Based on this idea, a category similarity mapping is introduced in order to measure the similarity between objects and clusters.

 {\bf Category similarity mapping:} \\
  $Sim$: $X\times \{\underline{X_1},\underline{X_2},\cdots,\underline{X_c}\} \mapsto R_{+}$ is called category similarity mapping if larger $Sim(x_k,\underline{X_i})$ means more similar between $x_k$ and $\underline{X_i}$, smaller $Sim(x_k,\underline{X_i})$ means less similar between $x_k$ and $\underline{X_i}$ .

 Similarly, category dissimilarity mapping can be defined as follows:

  {\bf Category dissimilarity mapping:} \\
  $Ds$: $X\times \{\underline{X_1},\underline{X_2},\cdots,\underline{X_c}\} \mapsto R_{+}$ is called category dissimilarity mapping if larger $Ds(x_k,\underline{X_i})$ means less similar between $x_k$ and $\underline{X_i}$, smaller $Ds(x_k,\underline{X_i})$ means larger similar between $x_k$ and $\underline{X_i}$.

According to the above analysis, a clustering result should be associated with a category similarity mapping  $Sim$ or a category dissimilarity mapping $Ds$. Usually, category (dis)similarity mapping is given by its corresponding clustering algorithm. For example, $\forall i \forall k ( Ds(x_k,\underline{X_i})= \| x_k-v_i\|^2)$  holds for C-means, where $\underline{X_i}=v_i$; $\forall i \forall k ( Sim(x_k,\underline{X_i})= max_{x_l \in X_i}s_{kl})$  holds for single linkage, where $\underline{X_i}=X_i$. Therefore, a clustering result can be represented by ($\underline{X},U,Sim$) or ($\underline{X},U,Ds$).

\section{Separation Axioms and Categorization Equivalency Principle for Clustering Results}

 Clearly, not all partitions correspond to a clustering result ($\underline{X},U,Sim$) or ($\underline{X},U,Ds$).What makes a clustering result different from a partition? At least, a clustering result should satisfy the elementary requirements of cluster analysis, but a partition does not have such constraints. In other words, the elementary requirements of cluster analysis can be considered as the common properties of clustering results, which means two facts:
 \begin{itemize}
   \item For each object, its intra cluster similarity should be greater than its inter cluster similarity. It means each object can be categorized into one cluster and only one cluster according to cluster similarity mapping. Otherwise, there exist one object belonging to two and more clusters, in other words, some objects in the different clusters are more similar than some objects in the same cluster, which contradicts the elementary requirements of cluster analysis as it requires that objects in the same cluster should be more similar, objects in different clusters should be more dissimilar. This fact is called sample separation.
   \item For each cluster, there exists at least one object, its intra cluster similarity with respect to this cluster should be greater than its inter cluster similarity in order to keep each cluster not null. Otherwise, there exists one cluster, no object belong to this cluster in essence. It also contradicts the elementary requirements of cluster analysis as it requires that a cluster in a clustering result at least has one object in order to make a partition. This fact is called
   cluster separation.
 \end{itemize}

Transparently, sample separation and category separation are self-evident for a clustering result and can be considered as axioms for clustering results. In mathematical terms, if ($\underline{X},U,Sim$) is a clustering result of the data set $X$, then it satisfies two separation axioms.

{\bf 1) Sample Separation Axiom:}
$\forall k \exists i$ $\forall j( (j\neq i) \rightarrow (Sim(x_k,\underline{X_i}) > Sim(x_k,\underline{X_j})))$.

{\bf 2) Category Separation Axiom:}
 $\forall i\exists k$ $\forall j((j\neq i) \rightarrow (Sim(x_k,\underline{X_i}) > Sim(x_k,\underline{X_j})))$.

 For a clustering result ($\underline{X},U, Sim$), ($\underline{X}, Sim$) result in a categorization of the dat set $X$ by separation axioms, and the corresponding partition $U$ also represents a categorization. For a clustering result, its cognitive category representation should have the same categorization ability as its extensional cluster representation in common sense. Therefore,  a clustering result ($\underline{X},U, Sim$) of the data set $X=\{x_1,x_x,\cdots,x_n\}$ should satisfy {\bf categorization equivalency axiom} as follows.

%{\bf Categorization Equivalency Principle:} \\
          $\forall k  ( \arg \max_i u_{ik} = \arg \max_i Sim(x_k,\underline{X_i}) )$.

Therefore,  a cognitive category representation is equivalent to a partition for a clustering result when categorizing an object if categorization equivalency axiom holds. In this sense, a clustering result can be represented by its corresponding cognitive category representation or its corresponding partition for the sake of brevity in practice. Sometimes, $\forall i \forall k (Sim(x_k,\underline{X_i})=u_{ik})$, categorization equivalency axiom becomes a tautology,  a clustering result can be represented by $U$.  Roughly speaking, separation axioms are more essential requirements for a clustering result.

When a clustering result of the data set $X$ is represented by ($\underline{X},U, Ds$), two separation axioms can be expressed as follows:

{\bf 1) Sample Separation Axiom:}
$\forall k \exists i$ $\forall j ((j\neq i) \rightarrow (Ds(x_k,\underline{X_i}) < Ds(x_k,\underline{X_j}))$).

{\bf 2) Category Separation Axiom:}
 $\forall i\exists k$ $\forall j  ((j\neq i) \rightarrow (Ds(x_k,\underline{X_i}) < Ds(x_k,\underline{X_j}))$).

Similarly,  {\bf categorization equivalent axiom} can be expressed as follows:

    $\forall k (\arg \max_i u_{ik} = \arg \min_i Ds(x_k,\underline{X_i}) )$.

Considered the equivalence between dissimilarity and similarity in some sense, a clustering algorithm usually uses one. In this paper, separation axioms are used to denote the above two cases.

Based on the above analysis,  several properties about clustering results can be proved such as Theorem \ref{Hardclustering} and \ref{propertiesclusterseparation}.

\begin{theorem}
\label{Hardclustering}
If a clustering result is $(\underline{X},U,Sim)$  where $U=[u_{ik}]_{c\times n}$, $u_{ik}\in \{0,1\}$ and categorization equivalency axiom holds, then  that $U$ is hard partition is equivalent to that separation axioms hold for $(\underline{X},U,Sim)$.
\end{theorem}
\begin{proof}
	If $U$ is a hard partition, then $\sum_{i=1}^{c} u_{ik}=1, u_{ik} \in \{0,1 \}, \forall i, 1\leq \sum_{k=1}^{n} u_{ik} < n$.
	$\sum_{i=1}^{c} u_{ik}=1, u_{ik} \in \{0,1 \}$ implies that $\forall k \exists i \forall j ((j\neq i)\rightarrow (u_{ik}>u_{jk}))$. According to categorization equivalency axiom, $\forall k \exists i \forall j ((j\neq i) \rightarrow( Sim(x_k,\underline{X_i}) > Sim(x_k,\underline{X_j})))$, which means sample separation axiom holds.$\forall i \forall k$ $u_{ik} \in \{0,1 \}, \forall i ( 1\leq \sum_{k=1}^{n} u_{ik} < n)$ follows that $\forall i \exists k \forall j ((j\neq i) \rightarrow(u_{ik}>u_{jk}))$. By categorization equivalency axiom, $\forall i \exists k \forall j ((j\neq i) \rightarrow( Sim(x_k,\underline{X_i}) > Sim(x_k,\underline{X_j})))$, which means category separation axiom holds.  Therefore, if  $U$ is a hard partition, separation axioms hold for $(\underline{X},U,Sim)$.
	
	Similarly, if separation axioms hold for $(\underline{X},U,Sim)$, then $U$ is a hard partition.
\end{proof}
\begin{theorem}
\label{propertiesclusterseparation}
If a clustering result $(\underline{X},U,Sim)$  satisfies category separation axiom, then we have two conclusions.
\begin{enumerate}
  \item $\forall i \forall j$   $((i\neq j)\rightarrow (\underline{X_i} \neq \underline{X_j}) )$.
  \item There exists at least $c$ objects $x_{k_i}$  such that $\forall i \forall j ((i\neq j)\rightarrow (x_{k_i} \neq x_{k_j})) $.
\end{enumerate}
\end{theorem}
\begin{proof}
	For a clustering result $(\underline{X},U,Sim)$, there exist $i\neq j$ such that $\underline{X_i}=\underline{X_j}$. . According to category separation axiom, for cluster $X_i$, there exists an object $x_{k_i}$ such that $Sim(x_{k_i},\underline{X_i}) > Sim(x_{k_i},\underline{X_j})$.  However, $\underline{X_i}=\underline{X_j}$ means that $Sim(x_{k_i},\underline{X_i}) = Sim(x_{k_i},\underline{X_j})$.  It is a contradiction. In other words,  the first conclusion is proved.
	
	Similarly, if category separation axiom  hold,  for cluster $X_i$ there exists an object $x_{k_i}$ such that $\forall j ((j\neq i)\rightarrow (Sim(x_{k_i},\underline{X_i}) > Sim(x_{k_i},\underline{X_j})))$.  If there exist $i\neq j$  such that $x_{k_i}=x_{k_j}$, then $Sim(x_{k_i},\underline{X_i}) > Sim(x_{k_i},\underline{X_j})$ and $Sim(x_{k_j},\underline{X_j}) > Sim(x_{k_j},\underline{X_i})$. Since $x_{k_i}=x_{k_j}$,
	it means that $Sim(x_{k_i},\underline{X_i}) > Sim(x_{k_j},\underline{X_j})$ and $Sim(x_{k_j},\underline{X_j}) > Sim(x_{k_i},\underline{X_i})$, which is a contradiction.  Therefore, the second conclusion is proved.  Hence, the proof is finished.
\end{proof}

\subsection{Classification of Clustering Results and Boundary Set}

The proposed separation axioms can classify clustering results into three types as follows:

{\bf Proper clustering:} If a clustering result follows  separation axioms, such a clustering result is proper.

{\bf Overlapping clustering:} If a clustering result obeys category separation axiom but violates sample separation axiom, it is called overlapping clustering.

{\bf Improper clustering:} If a clustering result violates category separation axiom, it is called improper clustering.

 Obviously, proper clustering and overlapping clustering are useful in practice. In cluster analysis community, overlapping clusters are usually taken as meaningful~\citep{jardine1971mathematical} and~\citep{ding2004k}, and have real applications such as categorization in cognitive science~\citep{shepard1979additive}, community detection in complex networks~\citep{palla2005uncovering}, movie recommendation~\citep{banerjee2005model}, etc.

 However, a good clustering result is not expected to be an improper clustering. In particular, a clustering algorithm is not expected to generate improper clustering results when the given data set has a well clustered structure. Improper clustering has different cases. Two special cases of improper clustering can be defined as follows.

 {\bf Coincident clustering:}  For $\underline{X}=\{\underline{X_1},\underline{X_2},\cdots,\underline{X_c}\}$, if $\exists i \exists j ((i\neq j) \wedge(\underline{X_i}=\underline{X_j} ))$, it can be called coincident clustering.

 {\bf Totally coincident clustering:} For $\underline{X}=\{\underline{X_1},\underline{X_2},\cdots,\underline{X_c}\}$, if $\forall i \forall j (\underline{X_i}=\underline{X_j}) $, it is called totally coincident clustering.

 By categorization equivalency axiom, we can offer a classification of partition according to the the classification of clustering results as follows:

 {\bf Proper partition:}  a partition $U=[u_{ik}]_{c\times n}$ is proper if $\forall k\exists i$$\forall j $ $((j\neq i)\rightarrow (u_{ik} > u_{jk}))$ and $\forall i\exists k$ $\forall j ((j\neq i)\rightarrow ( u_{ik} > u_{jk}))$.

 {\bf Overlapping partition:}  a partition $U=[u_{ik}]_{c\times n}$ is overlapping if $\exists k\exists  i \exists j ((j\neq i)\wedge
 (u_{ik}=u_{jk}=\max_{\iota} u_{\iota k} ))$ and $\forall i\exists k$ $\forall j ((j\neq i)\rightarrow (u_{ik} > u_{jk}))$.

 {\bf Improper partition:}  a partition $U=[u_{ik}]_{c\times n}$ is improper if $\exists i \forall k \exists j ((j\neq i)\wedge
 (u_{ik}\leq u_{jk} ))$ .

 More detailed, several special cases of improper partition can be defined as follows.

 {\bf Covering partition:} If a partition $U=[u_{ik}]_{c \times n}$ satisfies $\exists i \exists j ((i\neq j)\wedge \forall k (u_{ik}\leq u_{jk}))$, $U=[u_{ik}]_{c \times n}$  is called covering partition.

 {\bf Coincident partition:} If a partition $U=[u_{ik}]_{c \times n}$ satisfies $\exists i \exists j ((i\neq j) \wedge \forall k (u_{ik}=u_{jk}))$, $U=[u_{ik}]_{c \times n}$  is called coincident partition.

 {\bf Uninformative partition:} $U_{\pi}=[\pi_1,\pi_2,\cdots,\pi_c]^T \otimes {\bf 1}_{1\times n}$ is called uninformative partition, where $\otimes$ represents Kronecker product, ${\bf 1}$ denotes the vector of all 1's.

 {\bf Absolute uninformative partition:}  $U_{c^{-1}}=[c^{-1}]_{c\times n}$ is called absolute uninformative partition.

 In the literature, many clustering algorithms have been reported to produce improper clustering. For example, Possibilistic C-means has the undesirable tendency to produce coincident clustering~\citep{barni1996comments}, absolute uninformative partition has been pointed out to be a fixed point of EM for Gaussian Mixture~\citep{figueiredo2000unsupervised}, and similar result has been discovered by~\citep{newman2007mixture} when applying probabilistic mixture models and the expectation-maximization algorithm into detecting community, General C-means has coincident clustering with mild conditions, and many algorithms can output totally coincident clustering, which is usually the mass center of the given data set, details can be seen in~\citep{yu2005general}.

 According to the above analysis, separation axioms are independent, which can be proved as Theorem \ref{sample separation and cluster separation independ}.

\begin{theorem}
	\label{sample separation and cluster separation independ}
	For soft clustering,  separation axioms are independent of each other when categorization equivalency axiom holds.
\end{theorem}
\begin{proof}
	According to the previous analysis, uninformative partition and overlapping clustering correspond to some clustering results.
	For $U_{\pi}=[\pi_1,\pi_2,\cdots,\pi_c]^T \otimes {\bf 1}_{1\times n}$, if $\forall i\neq j, \pi_i\neq \pi_j$, then $U_{\pi}$ follows sample separation axiom but violates category separation axiom.
	%In , a coincident partition usually follows sample separation axiom but must violate cluster separation axiom.
	As for overlapping clustering, category separation axiom hold but violates sample separation axiom.  Therefore, the proof is finished.
\end{proof}

In order to further study overlapping clustering, we introduce the definition of boundary set as follows.

{\bf Boundary set:} If a data set with $n$ objects is clustered into $c$ groups for the category similarity mapping $Sim$: $X\times \{\underline{X_1},\underline{X_2},\cdots,\underline{X_c}\} \mapsto R_{+}$, the boundary set is defined as (\ref{simsepaxiom}).
\begin{equation}
\label{simsepaxiom}
B_{Sim}(X)=\{ x_{k} | \mid arg\max_{1\leq i\leq c} Sim(x_k,\underline{X_i})\mid>1 \} \end{equation}
where $\mid   X \mid$ represents the cardinality of a set $X$. Similarly, the boundary set can be defined by category dissimilarity mapping.

When the boundary set is not empty, sample separation axiom does not hold. In practice, the boundary set may be important sometimes~\citep{huang1999fuzzy}.  In theory,~\citep{ben2008relating} studied the relation between the cluster boundary and the stability of a clustering algorithm from a probability point of view.  Since different clustering algorithms have different category similarity mappings, the boundary set will depends on  the given data set and the clustering algorithm. In the future, the relation between overlapping clustering and boundary set will be further studied.

\subsection{Inequalities on Clustering Results}

Separation axioms not only classify clustering results in theory, but also follow several inequalities on clustering results such as Theorem \ref{Simprodsum} and \ref{dSimprodsum}.

\begin{theorem}
	\label{Simprodsum}
	Let $(\underline{X},U,Sim)$ be a clustering result for the given data set $X=\{x_1,x_2,\cdots,x_n\}$.
	If categorization axioms hold, the inequality (\ref{simprod}),  (\ref{simsum}),(\ref{simmixture}) and (\ref{simgeneralmeans}) hold.
	\begin{equation}\label{simprod}
	\prod_{k}Sim(x_k,\underline{X_{\varphi(k)}})\geq \prod_{k}Sim(x_k,\underline{X_{\phi(k)}})
	\end{equation}
	\begin{equation}\label{simsum}
	\sum_{k}Sim(x_k,\underline{X_{\varphi(k)}})\geq \sum_{k} Sim(x_k,\underline{X_{\phi(k)}})
	\end{equation}
	\begin{equation}\label{simmixture}
	\prod_{k}Sim(x_k,\underline{X_{\varphi(k)}})\geq \prod_{k}\sum_{i} \alpha_i Sim(x_k,\underline{X_i})
	\end{equation}
	\begin{equation}\label{simgeneralmeans}
	\sum_{k}Sim(x_k,\underline{X_{\varphi(k)}})\geq \sum_{k} f(\sum_{i}\alpha_{i}g(Sim(x_k,\underline{X_{i}})))
	\end{equation}
	where $\varphi(k)=arg\max_{i} u_{ik}$, $\phi(k)$ is a function from $\{1,2,\cdots, n\}$ to $\{1,2,\cdots,c\}$, $\alpha_i>0$ and $\sum_{i=1}^c\alpha_i=1$, $f$ is a convex function,  $\forall t\in R_{+},f(g(t))=t$.
\end{theorem}
\begin{proof}
	1) Since categorization axioms holds, $\varphi(k)=arg\max_{i}u_{ik}$ is well defined and $\forall k (\arg \max_i u_{ik} = \arg \max_i Sim(x_k,\underline{X_i}) )$, thus the inequality $Sim(x_k,\underline{X_{\varphi(k)}})\geq Sim(x_k,\underline{X_{\phi(k)}})\geq 0$ must hold. Therefore, the inequality (\ref{simprod})  can be easily proved by multiplying the above inequality according to subscript $k$ from 1 to n.
	
	2) Similarly, the inequality (\ref{simsum}) can be proved.
	
	3) As $Sim(x_k,\underline{X_{\varphi(k)}})\geq Sim(x_k,\underline{X_{i}})\geq 0, \forall i$ , therefore, we know that inequality (\ref{alphasimsep}) holds.
	\begin{equation}
	\label{alphasimsep}
	\alpha_i Sim(x_k,\underline{X_{\varphi(k)}})\geq \alpha_i Sim(x_k,\underline{X_{i}})\geq 0.
	\end{equation}
	
	By summarizing inequality (\ref{alphasimsep}) according to subscript $i$ from 1 to c, we can get inequality (\ref{sumalphasimsep}).
	\begin{equation}
	\label{sumalphasimsep}
	\sum_i \alpha_i Sim(x_k,\underline{X_{\varphi(k)}})\geq \sum_{i}\alpha_i Sim(x_k,\underline{X_{i}})\geq 0.
	\end{equation}

	As $\sum_{i=1}^c\alpha_i=1$, the inequality (\ref{simmixture})  can be easily proved by multiplying the inequality (\ref{sumalphasimsep}) according to subscript $k$ from 1 to n.
	
	4) As $f$ is a convex function, we know that inequality (\ref{generalmeanssimsep}) holds.
	\begin{equation}
	\label{generalmeanssimsep}
	\sum_{i}\alpha_{i}f(g(Sim(x_k,\underline{X_{i}})))\geq f(\sum_{i}\alpha_{i}g(Sim(x_k,\underline{X_{i}})))
	\end{equation}
	Since  $\forall t\in R_{+},f(g(t))=t$, we know inequality (\ref{generalmeanssimsep}) becomes inequality (\ref{genmeanssimsep}).
	\begin{equation}
	\label{genmeanssimsep}
	\sum_{i}\alpha_{i}Sim(x_k,\underline{X_{i}})\geq f(\sum_{i}\alpha_{i}g(Sim(x_k,\underline{X_{i}})))
	\end{equation}
	By similar analysis as above, we can get inequality (\ref{simgeneralmeans}).
	
	In summary, the proof of Theorem \ref{Simprodsum} is finished.
\end{proof}

By similar analysis of Theorem \ref{Simprodsum}, we can prove Theorem \ref{dSimprodsum}.

\begin{theorem}
	\label{dSimprodsum}
	Let $(\underline{X},U,Ds)$ be a clustering result for the given data set $X=\{x_1,x_2,\cdots,x_n\}$.
	If categorization axioms hold, the inequality (\ref{dsimsum}), (\ref{dsimmixture}),(\ref{dsimprod}) and (\ref{dsimmixprod}) hold.
	\begin{equation}\label{dsimsum}
	\sum_{k}Ds(x_k,\underline{X_{\varphi(k)}})\leq \sum_{k} Ds(x_k,\underline{X_{\phi(k)}})
	\end{equation}
	\begin{equation}\label{dsimmixture}
	\sum_{k}Ds(x_k,\underline{X_{\varphi(k)}})\leq \sum_{k}f(\sum_{i} \alpha_i g(Ds(x_k,\underline{X_i})))
	\end{equation}
	\begin{equation}\label{dsimprod}
	\prod_{k}Ds(x_k,\underline{X_{\varphi(k)}})\leq \prod_{k} Ds(x_k,\underline{X_{\phi(k)}})
	\end{equation}
	\begin{equation}\label{dsimmixprod}
	\prod_{k}Ds(x_k,\underline{X_{\varphi(k)}})\leq \prod_{k}f(\sum_{i} \alpha_i g(Ds(x_k,\underline{X_i})))
	\end{equation}
	where $\varphi(k)=arg\max_{i} u_{ik}$, $\phi(k)$ is a function from $\{1,2,\cdots, n\}$ to $\{1,2,\cdots,c\}$, $\forall t\in R_{+},f(g(t))=t$, and $f$ is a concave function, $\alpha_i>0$ and $\sum_{i=1}^c\alpha_i=1$.
\end{theorem}
Theorem \ref{Simprodsum} and \ref{dSimprodsum} gives some qualitatively properties of clustering results when categorization axiom hold. Clearly,  Theorem \ref{Simprodsum} and  \ref{dSimprodsum} show that the clustering result should reach the optimum of some functions, which gives some new interpretations of clustering algorithms. In Section 4.2, we will further discuss this issue.

\subsection{Related Works}

If $U$ is a hard partition and the clustering result $(\underline{X},U,Sim)$ satisfies sample separation axiom and categorization equivalency axiom, then objects belonging to the cluster $X_i$ have the common property: their similarities to the cognitive category representation $\underline{X_i}$ are the maximal. In common sense, objects are similar to each other just because they have some common property. In this respect, objects in the same cluster can be considered similar to each other with respect to that above common property.  From this point of view, sample separation axiom is consistent with Gestalt laws of grouping~\citep{wertheimer1923untersuchungen} as Gestalt laws of grouping state that objects should be grouped together if they are similar to each other.  Generally speaking, sample separation axiom explains quantitatively in what sense objects in the same cluster are similar to each other.

Moreover, sample separation axiom is consistent with two famous categorization theories in cognitive science: prototype theory of categorization~\citep{rosch1978principles}, exemplar theory of categorization~\citep{medin1978context}. Prototype theory of categorization tells us that one object is categorized as one cluster A not other clusters just because it is more similar to the prototype of cluster A than it is to those of other clusters when a cluster is represented by a prototype. Exemplar theory of categorization tells us that one object is categorized as one cluster A not other clusters just because it is more similar to multiple exemplars of cluster A than it is to those multiple exemplars of other clusters when a cluster is represented by multiple exemplars. When $U$ is a hard partition, prototype theory of categorization and exemplar theory of categorization are consistent with sample separation axiom. From the point of view in sample separation axiom, the difference between prototype theory and exemplar theory is cluster representation, and sample separation axiom seems to be a generalization of prototype theory and exemplar theory. Transparently, separation axioms present the common property that cluster results should satisfy.

In addition, some properties of category separation axiom have been pointed out in the literature. For example, when axiomatizing clustering function,~\citep{wright1973formalization} has taken it for granted that $c$ clusters at least have $c$ different objects when a given data set is divided into $c$ clusters, which is a property of category separation axiom.  When studying clustering techniques,~\citep{bonner1964some} though that it was desirable to have a core for each of the other cluster sets that are reasonably different from each other, which is also another property of cluster separation axiom. In particular,~\citep{wright1973formalization} pointed out a hard partition should satisfy mutually exclusive, totally exhaustive and nonempty, which can be guaranteed by separation axioms and categorization equivalency axiom as shown in Theorem \ref{Hardclustering}.

\section{Principles of Developing Clustering Algorithms}

The common properties of clustering results can offer some principles to develop a clustering algorithm. For a clustering result, it should satisfy categorization equivalency axiom, sample separation axiom and cluster separation axiom. Inversely, three principles of developing a clustering algorithm can be presented based on categorization equivalency axiom, sample separation axiom and cluster separation axiom. In the following, we will discuss such three principles respectively.

\subsection{Categorization Equivalency Principle}

According to categorization equivalency axiom, its cognitive category representation is equivalent to its extensional cluster representation for a clustering result with respect to $Sim$ or $Ds$. For a clustering result,  its extensional cluster representation can be generated by its cognitive category representation with category similarity mapping. Inversely, the cognitive category representation can be determined by the corresponding extensional cluster representation and the given data set. Therefore, a clustering result should be a convergence point for the above mentioned iterated process. Considered the extensional cluster representation is a partition, a common idea of developing a clustering algorithm is to alternatively update the partition and the cognitive category representation, such as~\citep{runkler1999alternating,macqueen1967some} and so on. When such iterated process converges, it is easy to know that categorization equivalency axiom holds for the outputted clustering result.  Hence,  a framework of iterative clustering algorithm is developed based on categorization equivalency axiom as Algorithm \ref{alg:clustering Framwork}.

\begin{algorithm}[H]
	\caption{ Framework of Iterative Clustering Algorithm.}
	\label{alg:clustering Framwork}
	\begin{algorithmic}[1] %
		\REQUIRE ~~\\
		$I(X)$ represents the data set $X$,\\
		Initial Partition $U$;\\
		Convergence Thresholds;\\
		Free Parameters;
		%\ENSURE ~~\\ Output $U$;
		\REPEAT
		\STATE  Update $\underline{X_1},\underline{X_2},\cdots,\underline{X_c} $ by a partition U;
		\label{code: partition representation}%
		\STATE Update U by cluster similarity mapping $Sim(x_k,\underline{X_i})$;
		\label{code: partition}
		\UNTIL {Convergence thresholds are satisfied, output U or $\{\underline{X_1},\underline{X_2},\cdots,\underline{X_c} \}$}% otherwise go to step \ref{code: partition representation};
		\label{code:stop}
		%\RETURN $E_n$
	\end{algorithmic}
\end{algorithm}

In the literature, many clustering algorithms can be described by Algorithm \ref{alg:clustering Framwork}. How to define a cognitive category representation and a cluster similarity mapping  is the core of developing a clustering algorithm. Different cluster representations and different category similarity mappings lead to diverse clustering algorithms.  For example, Gath-Geva clustering~\citep{gath1989unsupervised},  Alternating Cluster estimation~\citep{runkler1999alternating} are consistent with categorization equivalency axiom. As two examples of Algorithm \ref{alg:clustering Framwork}, we just rewrite Single linkage and C-means as Algorithm \ref{alg:Single Linkage} and \ref{alg:C-means}. More interestingly, C-means is a typical algorithm based on prototype theory of categorization, and single linkage is a typical algorithm based on exemplar theory of categorization.
\begin{algorithm}[H]
	\caption{ Single Linkage Clustering Algorithm.}
	\label{alg:Single Linkage}
	\begin{algorithmic}[1] %
		\REQUIRE ~~\\
		Similarity matrix $S(X)$ represents the data set $X$,\\
		Initial Partition $U=I_n$, where n is the number of objects;\\
		Convergence Thresholds: c denotes the number of clusters;\\
		%\ENSURE ~~\\ Output $U$;
		\REPEAT
		\STATE  Update $\underline{X_1}=X_1,\underline{X_2}=X_2,\cdots,\underline{X_c}=X_c $ by a partition U;
		%\label{code: partition representation}%
		\STATE Update U by merging $X_i$ and $X_j$ if $\{i,j\}=arg\max_{j\neq i}\max_{x_k\in X_j}Sim(x_k,\underline{X_i})$, where $Sim(x_k,\underline{X_i})=\max_{x_l\in X_i}s_{kl}$,;
		%\label{code: partition}
		\UNTIL {when the number of iterations reaches $n-c$, output U}% otherwise go to step \ref{code: partition representation};
		\label{code:stop}
		%\RETURN $E_n$
	\end{algorithmic}
\end{algorithm}

\begin{algorithm}[H]
	\caption{ C-means Clustering Algorithm.}
	\label{alg:C-means}
	\begin{algorithmic}[1] %
		\REQUIRE ~~\\
		Feature matrix $F(X)$ represents the data set $X$.\\
		Initial Partition $U$;\\
		Convergence Thresholds: IT denotes the maximum number of iterations;\\
		Free Parameter: c denotes the number of clusters
		%\ENSURE ~~\\ Output $U$;
		\REPEAT
		\STATE  Update $\underline{X_1}=mean(X_1)$, $\underline{X_2}=mean(X_2)$, $\cdots$, $\underline{X_c}=mean(X_c) $ by a partition U;
		%\label{code: partition representation}%
		\STATE Update U by $u_{ik}=1$ if $k=arg\min_{ i} Ds(x_k,\underline{X_i})$, otherwise $u_{ik}=0$, where $Ds(x_k,\underline{X_i})= \| x_k-\underline{X_i}\|$;
		%\label{code: partition}
		\UNTIL {when the number of iterations reaches IT, output U}% otherwise go to step \ref{code: partition representation};
		\label{code:stop}
		%\RETURN $E_n$
	\end{algorithmic}
\end{algorithm}

\subsection{Cluster Compactness Principle}

A clustering result should abide by sample separation axiom. However, sample separation axiom just is a minimum requirement for a clustering result. In theory, it is not enough for a good clustering result to just follow sample separation axiom in borderline. Theoretically,  a good clustering result should keep away from violating sample separation axiom as far as possible.  In other words, it is very important for every object that the intra cluster similarity should be greater than the inter cluster similarity as much as possible. As larger intra cluster similarity means less inter cluster similarity, it is natural for a good clustering result to make the intra cluster similarity as larger as possible.  Obviously, the larger intra cluster similarity, the more compact a clustering result. Consequently, when designing a clustering algorithm, cluster compactness principle can be stated as:

{\bf Cluster Compactness Principle:} A clustering algorithm should make its clustering result as compact as possible.

More detailed,  cluster compactness principle requires that the maximum similarity between one object and its corresponding cluster should be as larger as possible, which is also equivalent to the minimum dissimilarity between one object and its corresponding cluster should be as small as possible. In other words, the meaning of compactness in cluster compactness principle is maximal intra cluster similarity or minimal intra cluster variance.

Certainly, cluster compactness principle can be directly used to design an clustering criterion  when cluster compactness is defined. Likely, different requirements lead to different definitions of cluster compactness. A general definition of cluster compactness criterion can be defined as follows.

{\bf Cluster Compactness Criterion:} $J_{C}: \{X\}\times \{\{\underline{X_1},\underline{X_2},\cdots,\underline{X_c}\}| \forall i, \underline{X_i} \; \textrm{represents}\; X_i\} $ $\mapsto R_{+}$ is called cluster compactness criterion if the optimum of $J_{C}(X,\{\underline{X_1},\underline{X_2},\cdots,\underline{X_c}\})$ corresponds to the clustering result with the largest cluster compactness.

According to Theorem \ref{Simprodsum} and \ref{dSimprodsum}, some cluster compactness criteria can be developed in theory.  When $\forall i,\underline{X_{i}}$ is fixed,  $\prod_{k}Sim(x_k,\underline{X_{\varphi(k)}})$ can be considered as one kind of maximal intra cluster similarity where $\varphi(k)=arg\max_{i}Sim(x_k,\underline{X_{i}})$. By Theorem \ref{Simprodsum},  the right term in inequality (\ref{simprod}) can be considered as a cluster compactness measure for any partition, its maximum can reach the left term in inequality (\ref{simprod}).  In other words, the right term in inequality (\ref{simprod}) can be chosen as a clustering criterion. Interestingly, maximizing such a clustering criterion can directly lead to Classification maximum likelihood under assumptions as follows.

Let the objects in cluster $X_i$  follow the distribution $p(x, \underline{X_i})$, where $\underline{X_i}$ is the distribution parameter,
$ 1\leq i\leq c$  respectively. Set $Sim(x_k,\underline{X_i})=p(x_k, \underline{X_i})$, Theorem \ref{Simprodsum}  requires that the clustering criterion (\ref{CML}) reaches the maximum in order to make the clustering result satisfy sample separation axiom.
\begin{equation}\label{CML}
CML=\prod_{k}p(x_k,\underline{X_{\phi(k)}}),
\end{equation}
where $\phi(k)$ is a function from $\{1,2,\cdots, n\} $ to $\{1,2,\cdots,c\}$.

Therefore, maximizing $\prod_{k}p(x_k, \underline{X_{\phi(k)}})$ is equivalent to maximizing $\sum_{k}\log p(x_k,\underline{X_{\phi(k)}})$,
which results in Classification Maximum Likelihood approach for cluster analysis~\citep{celeux1993comparison} .

Similarly, setting $Sim(x_k,\underline{X_i})=p(x_k, \underline{X_i})$,  the inequality (\ref{simmixture}) of Theorem \ref{Simprodsum} can lead to maximum likelihood approach for model based clustering as follows.
\begin{equation}
\label{Maximum likelihood}
\prod_{k}p(x_k,\underline{X_{\varphi(k)}})\geq \prod_{k}\sum_{i} \alpha_{i}p(x_k,\underline{X_i})
\end{equation}

Obviously, the right term of the inequality (\ref{Maximum likelihood}) is the famous likelihood of mixture model~\citep{redner1984mixture}, which shows many model based clustering algorithms follow cluster compactness principle.

Likely, the right term of in inequality (\ref{dsimsum})  in Theorem \ref{dSimprodsum} represent intra cluster variance and can be considered as a clustering criterion. Minimizing such a clustering criterion can directly lead to C-means if setting $Ds(x, \underline{X_i})=\|x-v_i\|^2$, Theorem \ref{dSimprodsum} requires minimizing of the clustering criterion (\ref{CMeans}) as follows:
\begin{equation}\label{CMeans}
\min_{\phi}\sum_{k}\|x_k-v_{\phi(k)}\|^2,
\end{equation}
It is easily proved that the optimal $\phi$ for minimizing the clustering criterion (\ref{CMeans}) is
$\varphi(k)=arg\min_{i}\|x_k-v_i\|^2$.
Obviously, minimizing of the clustering criterion (\ref{CMeans}) is equivalent to minimizing the clustering criterion of C-means.
By the same method, Theorem \ref{dSimprodsum} also demands to minimize the clustering criterion of General C-means~\citep{yu2005general} as follows.
\begin{equation}\label{GCMeans}
\sum_{k}f(\sum_{i} \alpha_i g(\|x_k-v_{i}\|^2)),
\end{equation}

In addition, maximal intra cluster similarity is equivalent to minimal intra cluster variance in some cases.  In particular,
if setting $p(x,\underline{X_i})=\kappa\exp(\|x-v_i\|^2/\sigma)$, where $\kappa,\sigma$ are constant, negative log of the cluster criterion (\ref{CML}) of Classification Maximum Likelihood can be rewritten as (\ref{CMLCMeans}).
\begin{equation}\label{CMLCMeans}
\sum_{k}\|x_k-v_{\phi(k)}\|^2/\sigma -n\log\kappa
\end{equation}

Minimizing (\ref{CMLCMeans}) is equivalent to minimizing the clustering criterion (\ref{CMeans}) of C-means. Under such assumption, maximal intra cluster similarity is equivalent to minimal intra cluster variance.

Certainly, Theorem \ref{Simprodsum} and \ref{dSimprodsum} offer new explanations for C-means, General c-means, Classification maximum likelihood and model based clustering, but not all cluster compactness criteria can be inferred from Theorem \ref{Simprodsum} and \ref{dSimprodsum} such as~\citep{clauset2004finding}. In Table \ref{cluster compactness algorithms}, several clustering algorithms are developed based on cluster compactness principle, including C-means,K-modes~\citep{huang1998extensions}, fuzzy c-means,  General stochastic block model~\citep{shen2011exploring}, Model based clustering~\citep{fraley2002model}, CNM~\citep{clauset2004finding} and so on.

In addition, Theorem \ref{Simprodsum} and \ref{dSimprodsum} can lead to some new clustering criteria.  In the following, two new clustering algorithms are presented based on Theorem \ref{Simprodsum} in order to show that the proposed axioms are helpful to design  new clustering algorithms.

Let $f=x^m$  where  $m\geq 1$, the right term of inequality (\ref{simgeneralmeans}) in Theorem \ref{Simprodsum} can be used an the objective function (\ref{Sampleweightedsim}) of a new clustering algorithm as follows:

\begin{equation}\label{Sampleweightedsim}
\sum_{k} (\sum_{i}\alpha_{i}Sim(x_k,\underline{X_i})^{\frac{1}{m}})^m
\end{equation}
In theory, maximizing (\ref{Sampleweightedsim}) can lead to new different clustering algorithms depending on how to define $Sim(x_k,\underline{X_i})$.

Assume the data set $X$ can be represented by a feature matrix. Set $Sim(x_k,\underline{X_i})=\exp(-\beta^{-1}\|x_k-v_i\|^2)$, then (\ref{Sampleweightedsim}) becomes (\ref{Sampleweightedgauss}).
\begin{equation}\label{Sampleweightedgauss}
\sum_{k} (\sum_{i}\alpha_{i}\exp(-(m\beta)^{-1}\|x_k-v_i\|^2))^m
\end{equation}
The update equations of maximizing (\ref{Sampleweightedgauss}) can be written as follows:
\begin{equation}\label{updatevSampleweightedgauss}
v_i=\frac{\sum_k a_k u_{ik}x_k}{\sum_k a_k u_{ik}}
\end{equation}
\begin{equation}\label{updatealgSampleweightedmm}
\alpha_i=\frac{\sum_k a_k u_{ik}}{\sum_i\sum_k a_k u_{ik}}
\end{equation}
\begin{equation}\label{updatemSampleweightedgauss}
u_{ik}=\frac{\alpha_i\exp(-(m\beta)^{-1}\|x_k-v_i\|^2))}{\sum_i\alpha_i\exp(-(m\beta)^{-1}\|x_k-v_i\|^2))}
\end{equation}
\begin{equation}\label{updateSampleweightedgauss}
a_k=(\sum_{i}\alpha_{i}\exp(-(m\beta)^{-1}\|x_k-v_i\|^2))^m
\end{equation}

Assume the data set $X$ can be represented by an adjacency matrix $A=[A_{kl}]_{n\times n}$. Set $Sim(x_k,\underline{X_i})=\prod_l\theta_{il}^{A_{kl}}$, where $\sum_l \theta_{il}=1$, then (\ref{Sampleweightedsim}) becomes (\ref{Sampleweightedmm}).
\begin{equation}\label{Sampleweightedmm}
\sum_{k} (\sum_{i}\alpha_{i}\prod_l \theta_{il}^{\frac{A_{kl}}{m}})^m
\end{equation}

Set $\forall k, d_k=\sum_l A_{kl}$, the update equations of maximizing (\ref{Sampleweightedmm}) can be written as follows:
\begin{equation}\label{updatevSampleweightedmm}
\theta_{il}=\frac{\sum_k a_k u_{ik}A_{kl}}{\sum_k a_k u_{ik}d_k}
\end{equation}
%where $d_k=\sum_l A_{kl}$
\begin{equation}\label{updatealmSampleweightedmm}
\alpha_i=\frac{\sum_k a_k u_{ik}}{\sum_i\sum_k a_k u_{ik}}
\end{equation}
\begin{equation}\label{updatemSampleweightedmm}
u_{ik}=\frac{\alpha_i\prod_l\theta_{il}^{\frac{A_{kl}}{m}}}{\sum_i\alpha_i\prod_l\theta_{il}^{\frac{A_{kl}}{m}}}
\end{equation}
\begin{equation}\label{updateSampleweightedmm}
a_k=(\sum_{i}\alpha_{i}\prod_l\theta_{il}^{\frac{A_{kl}}{m}})^m
\end{equation}

In the future, we will study more inequalities based on sample separation axiom and design more new clustering algorithms.

\begin{table*}[t]
	\caption{Clustering algorithms with different cluster compactness criteria}
	\label{cluster compactness algorithms}
	\vskip 0.15in
	\begin{center}
		\begin{small}
			%\begin{sc}
			\begin{tabular}{p{7.5cm}p{4.5cm}l}
				\hline
				%\abovespace\belowspace
				Algorithm & Cluster compactness criterion & Optimal value\\
				\hline
				%\abovespace
				C-means~\citep{macqueen1967some}&	 Sum of Squared errors	&  Minimum\\
				Fuzzy c-means	&  Sum of weighted Squared errors 	& Minimum\\
				Single Linkage~\citep{sneath1957application}	& minimum distance between clusters	& Minimum\\
				Complete Linkage~\citep{Srensen1948}& furthest distance between clusters	& Minimum\\
				CNM~\citep{clauset2004finding} &  modularity	& Maximum\\
				General stochastic block model~\citep{shen2011exploring}& likelihood	& Maximum\\
				model based clustering~\citep{fraley2002model} & likelihood	& Maximum\\
				support vector clustering~\citep{ben2002support}& enclosing sphere of Radius  	& Minimum \\
				Nonnegative Matrix Factorization~\citep{xu2003document} & Intra-cluster approximation 	&  Minimum \\
				Addtive clustering~\citep{shepard1979additive}& Intra-cluster approximation 	&  Minimum \\
				Information Based clustering~\citep{slonim2005information} &	Intra-cluster similarity & Maximum  \\
				pairwise data clustering~\citep{hofmann1997pairwise} &	Intra-cluster variance & Minimum \\
				dominant set clustering~\citep{pavan2007dominant} &	Intra-cluster similarity & Maximum \\
				\hline
			\end{tabular}
			%\end{sc}
		\end{small}
	\end{center}
	\vskip -0.1in
\end{table*}

\subsection{Cluster Separation Principle}

If a clustering result $(\underline{X},U,Sim)$ satisfies category separation axiom, then $\forall 1\leq i\neq j\leq c,\underline{X_i} \neq \underline{X_j}$.  For a good clustering result, it is not enough to just satisfy $\forall 1\leq i\neq j\leq c,\underline{X_i} \neq \underline{X_j}$.  Usually, the distance among clusters is expected to be as larger as possible.

Therefore, a clustering criterion should make its clustering result follow cluster separation principle as follows:

{\bf Cluster Separation Principle:} A good clustering result should have the maximum distance among clusters.

Cluster separation principle means that the distance among clusters need to be defined by a clustering algorithm and the distance among clusters should be as larger as possible, which follows that the partition outputted by an algorithm keep away from violating cluster separation axiom as far as possible.

According to the above analysis, it is very natural that a clustering algorithm makes its clustering result satisfy cluster separation axiom. Cluster separation principle requires to develop an clustering criterion of measuring cluster separation.

A general definition of cluster separation criterion can be defined as follows.

{\bf Cluster Separation Criterion:} \\
$J_{S}: \{X\}\times \{\{\underline{X_1},\underline{X_2},\cdots,\underline{X_c}\}| \forall i, \underline{X_i} \; \textrm{represents}\; X_i\} \mapsto R_{+}$ is called cluster separation criterion if the optimum of $J_{S}(X,\{\underline{X_1},\underline{X_2},\cdots,\underline{X_c}\})$ corresponds to the clustering result with maximal cluster separation.

Different applications usually leads to different cluster separation criteria.  Cluster separation criteria can be considered as clustering criteria. For example, suppose that the data set $X$ is represented by a similarity matrix $S(X)$, and  $\forall i, \underline{X_i}=X_i$ and $c=2$,  cut can be defined as (\ref{Minimum cut})~\citep{wu1993optimal}.  Cluster separation criterion requires  to minimize (\ref{Minimum cut}), which is consistent with minimum cut algorithm.
\begin{equation}\label{Minimum cut}
Cut=\sum_{k\in X_1}\sum_{l\in X_2}s_{kl},
\end{equation}

In the literature, the clustering criteria of some clustering algorithms are indeed developed based on cluster separation principle, such as clustering algorithms in Table \ref{cluster separation algorithms}, including Girvan and Newman's algorithm  (GN algorithm)~\citep{girvan2002community}, ratio cut~\citep{wei1989towards}, normalized cut~\citep{shi2000normalized},Maximum Margin Clustering~\citep{ben2002support}, and so on.

\begin{table*}[t]
	\caption{Clustering algorithms with different cluster separation criteria}
	\label{cluster separation algorithms}
	\vskip 0.15in
	\begin{center}
		\begin{small}
			%\begin{sc}
			\begin{tabular}{lll}
				\hline
				%\abovespace\belowspace
				Algorithm &Cluster separation criterion  & Optimal value\\
				\hline
				%\abovespace
				GN algorithm~\citep{girvan2002community} &	Betweenness & Maximum  \\
				Minimum cut~\citep{wu1993optimal}&	Cut 	& Minimum  \\
				Ratio cut~\citep{wei1989towards}&	Ratio Cut 	& Minimum  \\
				Normalized cut~\citep{shi2000normalized}	& Normalized cut	& Minimum \\
				Maximum margin clustering~\citep{xu2004maximum}& Margin	&Maximum \\
				\hline
			\end{tabular}
			%\end{sc}
		\end{small}
	\end{center}
	\vskip -0.1in
\end{table*}

\subsection{On Clustering Principles}

In the literature, many authors tried to define cluster by internal cohesion (homogeneity) and external
isolation (separation),  as~\citep{Everitt2011cluster} have declared. For example,~\citep{jain2010data} said that an ideal cluster can be defined as a set of points that is compact and isolated, more detailed,
~\citep{bonner1964some}  defined tight clusters by maximal complete subgraphs of the similarity matrix graph;~\citep{abell1958distribution} defined compact cluster as at least 50 members are within a given radial distance of the cluster center. Transparently, the idea of cluster compactness and cluster separation has been used to develop clustering algorithm in the literature, but most definitions about compactness or separation are originated from intuition and applications. In this paper, cluster compactness principle and cluster separation principle have been clearly followed by separation axioms. Moreover, some cluster compactness criteria can be followed by sample separation axiom.

In addition, cluster separation principle has something to do with cluster compactness principle. In special case,it can be proved that cluster separation principle are equivalent to cluster compactness principle. For example, a famous result about C-means can be rewritten as follows.

Let $X=\{x_1,x_2,\cdots,x_n\}$  be the given data set, and $\forall k, x_k \in R^r$, where  $r \in N$, and $r>0$, and $c$ groups are represented by $v_1,v_2,\cdots,v_c$  respectively, the compactness of cluster $X_i$ can be defined as  $\Sigma_{k}u_{ik}\|x_k-v_i\|^2$. According to cluster compactness principle, a good partition should correspond to minimizing the sum of the compactness of each cluster as  $\Sigma_i\Sigma_{k}u_{ik}\|x_k-v_i\|^2$.

According to cluster separation principle, a good partition should correspond to maximizing the distance among clusters.  If the distance among clusters is defined as $\Sigma_i\Sigma_{k}u_{ik}\|v_i-\overline{x}\|^2$, where $v_i=\frac{\Sigma_{k}u_{ik}x_k}{\Sigma_{k}u_{ik}}$, $\overline{x}=\frac{\Sigma_{k}x_k}{n}$.  Transparently, maximizing the above distance among clusters  indeed is far away from violating separation axioms as $\forall i, v_i=\overline{x}$ is a totally coincident clustering and violates separation axioms.
It is well known that $\sum_i\sum_{k}u_{ik}\|v_i-\overline{x}\|^2+\sum_i\sum_{k}u_{ik}\|x_k-v_i\|^2=\sum_{k}\|x_k-\overline{x}\|^2$ where $\sum_{i=1}^{c} u_{ik}=1, u_{ik} \in \{0,1 \}, \forall i, 1\leq \sum_{k=1}^{n} u_{ik} < n$, details can be seen in page 95 of the book~\citep{AKJian83bookclustering}.

Therefore,
maximizing the distance among clusters $\Sigma_i\Sigma_{k}u_{ik}\|v_i-\overline{x}\|^2$ is equivalent to minimizing  the sum of the compactness of each cluster as  $\Sigma_i\Sigma_{k}u_{ik}\|x_k-v_i\|^2$. In other words, cluster compactness principle is equivalent to cluster separation principle under the above circumstance.

However, cluster separation principle is usually not equivalent to cluster compactness principle. Sometimes, cluster separation criterion and cluster compactness criterion can be combined into a clustering criterion in the literature.  For example,~\citep{ozdemir2001fuzzy} proposed the ICS algorithm, its clustering criterion can be expressed as:
$$
J_{ICS}(U,V)=\frac{1}{n}\sum_i\sum_{k}u_{ik}^m\|v_i-x_k\|^2-\frac{\gamma}{c}\sum_t\|v_i-v_t\|^2
$$
Clearly, the first term of the clustering criterion of ICS reflects cluster compactness principle, the second term of the clustering criterion of ICS is consistent with cluster separation principle.

Generally speaking, when designing a clustering criterion, cluster separation principle and cluster compactness principle are seldom used simultaneously,  as it has much more computational complexity when considering cluster separation principle and cluster compactness principle together.

\section{Separation Axioms and Cluster Validity}

According to~\citep{LiuLiXiong2013IEEETC}, when designing a cluster validity index, two ideas are widely accepted,one is to measure how distinct or well separated a cluster is from other clusters (which is called separation criterion),  consistent with our proposed category separation axiom; the other is to measure how compact a cluster itself is (which is called compact criterion), consistent with sample separation axiom. Such two criteria are consistent with cluster separation principle and cluster compactness principle.

Furthermore, a clustering result violating separation axioms can not be considered to have the best clustering quality. Consequently, clustering results close to those improper clustering results should be regarded as not as good as those far from the improper clustering results. Such a view is useful for designing a cluster validity index, which leads to another principle of designing a cluster validity index based on separation axioms as follow:

{\bf Extreme Value Principle:}  A good cluster validity index should evaluate improper clustering results as the poorest clustering quality.

In the literature, there exists many cluster validity indices only consistent with one or two of the above three principles. For instance, partition coefficient $V_{pc}=\frac{1}{n}\sum_i\sum_k u_{ik}^2$~\citep{bezdek1974cluster} and partition entropy  $V_{pe}=\frac{1}{n}\sum_i\sum_k u_{ik}logu_{ik}$~\citep{bezdek1975mathematical} only reflect cluster compactness principle and Extreme Value Principle. The fuzzy hypervolume validity~\citep{gath1989unsupervised} only reflects cluster compactness principle. However, different from clustering criterion design, many cluster validity indices in Table \ref{validity_indices} are usually developed based on the above three principles together.

\begin{table*}[t]
	\caption{Several cluster validity indices based on separation axioms}
	\label{validity_indices}
	\vskip 0.15in
	\begin{center}
		\begin{small}
			%\begin{sc}
			\begin{tabular}{lp{2.5cm}}
				\hline
				%\abovespace\belowspace
				Index & Optimal number of clusters\\
				\hline
				%\abovespace
				\\
				$Kwon=\frac{\sum_{i=1}^{n} \sum_{k=1}^{n} u_{ik}^2 \| x_k-v_i \|^2+c^{-1}\sum_{i=1}^{c} \|v_i-\overline{x}\|^2}{n \times \min_{i \neq j} \| v_i-v_j \|^2}$~\citep{kwon1998cluster}
				&Minimum \\
				\\%\belowspace

				$V_P=\frac{1}{n}(\sum_{k} \max(u_{ik})-\frac{2}{c(c-1)}\sum_{i=1}^{c-1}\sum_{j=i+1}^c \sum_k\min(u_{ik},u_{jk}))$~\citep{chen2001rule}
				&Maximum\\
				%\belowspace
				\\
				$DB=\frac{1}{nc}\sum_i\max_{j\neq i}\{\frac{\sum_{x_k\in X_i}d(x_k,v_i)}{d(v_i,v_j)\times n_i}+\frac{\sum_{x_k\in X_j}d(x_k,v_j)}{d(v_i,v_j)\times n_j}\}$~\citep{davies1979cluster}
				&Minimum\\
				\\
				$FS=\sum_{i=1}^{c} \sum_{k=1}^{n} u_{ik}^m \| x_k-v_i \|^2- \sum_{i=1}^{c} \sum_{k=1}^{n} u_{ik}^m\| v_i-\overline{v} \|^2$
				~\citep{fukuyama1989new}
				&Minimum\\
				\\
				$Silhouette=\frac{1}{nc}\sum_i \{\sum_{x_k\in X_i}\frac{b(x_k)-a(x_k)}{n_i\times \max (b(x_k),a(x_k)) }\}$
				~\citep{rousseeuw1987silhouettes}
				&Maximum \\
				\\
				%\belowspace
				$I=\frac{\max_{i, j}d(v_i,v_j)\times \sum_k d(x_k,\overline{x})}{nc\times \sum_i\sum_{x_k \in X_i} d(x_k,v_i)}$
				~\citep{maulik2002performance}
				&Maximum\\
				%\belowspace
				\\
				
				$Dunn=\min_i\min_j\frac{\min_{x_k\in X_i, x_l\in X_j d(x_k,x_l)}}{\max_i{\max_{x_k\ x_l\in X_i} d(x_k,x_l)}}$
				~\citep{dunn1974well}
				&Maximum\\
				%\belowspace
				\\
				$CH(X,v,c)=\frac{(n-c)\sum_{i=1}^{c} \sum_{k=1}^{n} u_{ik}^2 \| v_i-\overline{x} \|^2}{(c-1)\sum_{i=1}^{c} \sum_{k=1}^{n} u_{ik}^2 \| v_i-x_k \|^2}$
				~\citep{calinski1974dendrite}
				&Maximum\\
				
				\hline
			\end{tabular}
			%\end{sc}
		\end{small}
	\end{center}
	\vskip -0.1in
\end{table*}

As an example, we will show that Xie-Beni Index indeed is consistent with the proposed principles of designing a cluster validity index~\citep{xie1991validity}.

Xie-Beni Index is a well known cluster validity index for fuzzy C-means, which is defined as (\ref{XBindex}).
\begin{equation}
\label{XBindex}
XB(X,U,V)=\frac{\sum_{i=1}^{c} \sum_{k=1}^{n} u_{ik}^2 \| x_k-v_i \|^2}{n \times \min_{i \neq j} \| v_i-v_j \|^2}
\end{equation}

The larger $XB(X,U,V)$, the worse clustering result $(U,V)$. The smaller $XB(X,U,V)$, the better clustering result $(U,V)$. Transparently, coincident partition will make $XB(X,U,V)$  approach infinity and thus should be considered as an improper  clustering result.  In $XB(X,U,V)$, the numerator represents the compactness of the partition,the denominator represents the separation degree of the partition. Obviously,$XB(X,U,V)$ simultaneously considers cluster compactness, cluster separation and extreme value principle.  Xie-Beni index was not originated from separation axioms but is naturally consistent with our proposed separation axioms.  It needs to point out that many recently developed cluster validity indices are consistent with the proposed principles, for example, CVNN~\citep{LiuLiXiong2013IEEETC} considers cluster compactness and cluster separation together.

\section{Discussions and Conclusions}

In this paper, clustering results are axiomatized according to elementary requirements of cluster analysis. It it is pointed out that a clustering result should satisfy categorization equivalency principle, sample separation axiom and cluster separation axiom by formalizing the representation of clustering results, otherwise, the elementary requirements of cluster analysis are violated. The proposed separation axioms not only classify clustering results into proper clustering, overlapping clustering and  improper clustering, but also classify partition into proper partition, overlapping partition and improper partition. Improper clustering includes coincident clustering, totally coincident clustering. Improper partition includes  covering partition, coincident partition, uninformative partition and absolute uninformative partition.  What's more, several inequalities on clustering results are also obtained based on sample separation axiom. In addition,  sample separation axioms are consistent with two categorization theories in cognitive science.

Based on the common properties of clustering results, we proposed three principles of developing a clustering algorithm, such as categorization equivalency principle, cluster compactness principle and cluster separation principle. Based on categorization equivalency principle, a general framework of iterative clustering algorithm is proposed based on the iteration between cognitive category representation and extensional clustering representation. Based on sample separation axiom, cluster compactness principle not only offers some new interpretations for C-means, Model based clustering, but also results in new clustering algorithms. Based on category separation axiom, cluster separation principle is proposed to design clustering criterion. Many popular clustering algorithms follow the proposed principles, although they are not originally developed based on the proposed axioms.  Moreover, a good clustering result should be kept far away from violating the proposed separation axioms. Such idea leads to three principles of designing a cluster validity index, including cluster separation principle, cluster compactness principle and extreme value principle, but such three principles are usually taken into account together for designing a cluster validity index.

In the future, we need to investigate whether the proposed axioms can introduce more interesting properties about clustering results like Theorem \ref{Simprodsum} and \ref{dSimprodsum}, and develop some new clustering algorithm or cluster validity index. Furthermore, it is worth  further study whether or not it is feasible for axiomatizing machine learning.

\subsubsection*{Acknowledgements}

We thank Xinbo Gao, Wensheng Zhang, Baogang Hu, Jiangshe Zhang and Liping Jing, their valuable discussions and suggestions have greatly improved the presentation of this paper. This work was supported by the NSFC grants (61033013,61370129), Ph.D Programs Foundation of Ministry of Education of China(20120009110006), PCSIRT(IRT 201206), Beijing Committee of Science and Technology, China(Grant No. Z131110002813118).

%\subsubsection*{References}
\bibliographystyle{apa}
\bibliography{bibfile}

\end{document}